\documentclass[twoside,11pt]{article}

\usepackage{jmlr2e}

\usepackage{graphicx}
\usepackage{subfigure}
\usepackage{amssymb}
\usepackage{amsmath}
\usepackage{amsbsy}
\usepackage{algorithm}
\usepackage{algorithmic}
\usepackage{multirow}

\ShortHeadings{Streaming View Learning}{}
\firstpageno{1}

\begin{document}

\title{Streaming View Learning}

\author{\name Chang Xu \email xuchang@pku.edu.cn \\
 \addr  Key Laboratory of Machine Perception (Ministry of Education)\\
       School of Electronics Engineering and Computer Science\\
       Peking University\\
        Beijing 100871, China
      \AND
       \name Dacheng Tao \email dacheng.tao@uts.edu.au \\
       \addr  Centre for Quantum Computation and Intelligent Systems\\
       Faculty of Engineering and Information Technology\\
       University of Technology, Sydney\\
       Sydney, NSW 2007, Australia
       \AND
       \name Chao Xu \email xuchao@cis.pku.edu.cn\\
       \addr  Key Laboratory of Machine Perception (Ministry of Education)\\
       School of Electronics Engineering and Computer Science\\
       Peking University\\
        Beijing 100871, China
}

%\editor{xx}

\maketitle

\begin{abstract}
An underlying assumption in conventional multi-view learning algorithms is that all  views can be simultaneously accessed. However, due to various factors when collecting and pre-processing  data from different views, the streaming view setting, in which  views arrive in a streaming manner, is becoming more common. By assuming that the subspaces of a multi-view model trained over past views are stable, here we  fine tune their combination weights such that the well-trained multi-view model  is compatible with new views. This largely overcomes the burden of learning new view functions and updating past view functions. We theoretically examine convergence issues and the influence of streaming views in the proposed algorithm.  Experimental results on real-world datasets suggest that studying the streaming views problem in multi-view learning is significant and that the proposed algorithm can effectively handle streaming views in different applications.
\end{abstract}


\section{Introduction}

In this era of exponential information growth, it is now possible to collect abundant data from different  sources to perform diverse tasks, such as social computing, environmental analysis, and disease prediction. These data are usually heterogeneous and possess distinct physical properties such that they can be categorized into different groups, each of which is then regarded as a particular view in multi-view learning. For example, in video surveillance \citep{wang2013intelligent}, placing multiple cameras at different positions around one area might enable better surveillance of that area in terms of accuracy and reliability. In another example, to accurately recommend  products to target customers \citep{jin2005maximum}, it is necessary to comprehensively describe the product by its image, brand, supplier, sales history, and user feedback. Effective descriptors have already been developed for object motion recognition: (i) histograms of oriented gradients (HOG) \citep{dalal2005histograms}, which focus on static appearance information; (ii) histograms of optical flow (HOF) \citep{laptev2008learning}, which capture absolute motion information; and (iii) motion boundary histograms (MBH) \citep{dalal2006human}, which encode related motion between pixels.

Rather than requiring that all the examples should be comprehensively  described based on each individual view, it might be better to exploit the connections and differences between multiple views to better represent examples. A number of multi-view learning algorithms \citep{blum1998combining,lanckriet2004learning,jia2010factorized,chen2012large,chaudhuri2009multi} have thus emerged to effectively fuse multiple features from different views. These have been widely applied to various computer vision and intelligent system problems.

Although an optimal description of the data might be obtained by integration of multiple views,  in practice it is difficult to guarantee that all the candidate views can be  accessed simultaneously. For example, establishing a camera network for video surveillance is a huge project that takes time to realize. The number of views used in tracking and detection has increased. Newly developed recommendation systems might well have  their images and text descriptions in place, but they  require a period of time to accumulate sales  and, therefore,  user feedback, which  are  key factors influencing the decisions of prospective customers. Images can be depicted by diverse visual features with distinct acquisition costs. For example, a few milliseconds might be sufficient  to extract the color histogram or SIFT descriptors from a normal-sized image, but time-cost clustering and mapping processes are further required to generate bag-of-word (BoW) features from SIFT descriptors.  Recent deep learning methods need longer time (usually hours or even days) to obtain a reasonable model for image feature extraction.

Conventional multi-view learning algorithms \citep{cai2013multi,kumar2011learning} have been developed in ideal settings in which all the views  are  accessed simultaneously. The real world, however, presents a more challenging multi-view learning scenario formed from multiple streaming views.  Newly arrived views might contain fresh and timely information, that are beneficial for further improving the multi-view learning performance. To make existing multi-view learning methods applicable to this streaming view setting, a naive approach  might be to treat  a new view arriving as a new stage each time and then running the multi-view learning algorithms again with the new views. However, this approach is likely to suffer from intensive computational costs or serious performance degradation. In contrast, here we propose an effective streaming view learning algorithm that assumes the view function subspaces in the well-trained multi-view model over sufficient past views are stable and fine tunes their combination weights for an efficient model update. We provide theoretical analyses to support the feasibility of the proposed algorithm in terms of convergence and estimation error. Experimental results in real-world clustering and classification applications demonstrate the practical significance of investigating streaming views in the multi-view learning problem and the effectiveness of the proposed algorithm.

\section{Problem Formulation}

In the standard multi-view learning setting, we are provided with $n$ examples of $m$ views $\{(x_{i}^{1}, \cdots,  x_{i}^{m})\}_{i=1}^{n}$, where $x_{i}^{v}\in\mathbb{R}^{D_{v}}$ is the feature vector on the $v$-th view of the $i$-th example. The feature matrices on different views are thus denoted as $\{X^{1}, \cdots, X^{m}\}$, where $X^{v}\in \mathbb{R}^{D_{v}\times n}$. Subspace-based multi-view learning approaches aim to discover a subspace shared by multiple views, such that the information from multiple views can be integrated in that subspace:
\begin{equation}
(f^1, \cdots, f^m): (x_{i}^1, \cdots, x_{i}^{m})\rightarrow z_{i},
\end{equation}
where $f^{v}$ is the view function on the $v$-th view, and $z_{i}\in \mathbb{R}^{d}$ is the latent representation in the subspace $\mathcal{Z}$. Based on the unified representation $z$ of the multi-view example, the subsequent tasks, including classification, regression, and clustering, can  easily be  accomplished. 

Most existing multi-view learning algorithms explicitly assume that all views $\{X^{1}, \cdots, X^{m}\}$ are static and can be simultaneously accessed for  multi-view model learning. If  new views $\{X^{m+1}, \cdots, X^{m+k}\}$ are  provided, the question arises of how to upgrade the well-trained multi-view model $(f^{1}, \cdots, f^{m})$ over the past $m$ views using the latest information. It is unreasonable to simply neglect the newly arrived views and ignore the possibility of updating  the model. On the other hand, naively maintaining a training pool composed of all views, enriching the pool with each newly arrived view, and then re-launching the multi-view learning algorithm will be resource (storage and computation) consuming. It is, therefore, necessary to investigate this challenging multi-view learning problem in the general streaming view setting, where multiple views arrive in a streaming format. 

\subsection{Streaming View Learning}
In this section, we first present a naive approach to handle new views. We then develop a sophisticated streaming view learning algorithm that reduces the burden of learning new views.

\subsubsection{A Naive Approach}
Assume that multiple example views $\{x^{1}, \cdots, x^{m}\}$ are generated from a latent data point $z$ in the subspace, 
\begin{equation}\label{eq:view_func}
x^{v} = f^{v}(z) = W_{v}z,
\end{equation}
where view function $f^{v}$ parameterized by $W_{v}\in\mathbb{R}^{D_{v}\times d}$ can be assumed to be linear for simplicity. In practice, different feature dimensions are usually correlated; for example, distinct image tags  in BoW features might be related to each other. It is thus reasonable to encourage a low rank of $W_{v}$. Moreover, the low-rank $\{W_{v}\}_{v=1}^{m}$ implies that the latent subspace contains comprehensive information to generate multiple view spaces while the inverse procedure is infeasible, which is consistent with our assumption that multiple views are generated from a latent subspace.

Within the framework of empirical risk minimization,  view functions $W=\{W_{1}, \cdots, W_{m}\}$ can be solved with the following problem:
{\small
\begin{equation}\label{eq:batch}
\min_{W, z} \; \frac{1}{nm}\sum_{i=1}^{n}\sum_{v=1}^{m}\|x_{i}^{v}-W_{v}z_{i}\|_{2}^{2} + C_{1}\sum_{v=1}^{m}\|W_{v}\|_{*} + C_{2}\sum_{i=1}^{n}\|z_{i}\|_{2}^{2},
\end{equation}
}\noindent
where a least-squared loss is employed to measure the reconstruction error of multi-view examples, and a trace norm is applied to regularize view functions.

Suppose that, by solving Problem (\ref{eq:batch}), we have already a well-trained multi-view model $\{W_{1}, \cdots, W_{m}\}$ over $m$ views $\{X^{1}. \cdots, X^{m}\}$. Considering a new arriving view $X^{m+1}$, we are then faced with challenging problems of  how to discover the view function $W_{m+1}$ for the new view and how to upgrade the view functions $\{W_{1}, \cdots, W_{m}\}$ on past $m$ views. It is a straightforward extension to simultaneously handle more than one new view. 

Within the framework of Eq. (\ref{eq:batch}), the latent representations $\{z_{i}\}_{i=1}^{n}$ previously learned over $m$ past views are regarded as fixed. Then, the new view function $W_{m+1}$ can be efficiently solved by 
\begin{equation}
\min_{W_{m+1}} \; \frac{1}{n}\sum_{i=1}^{n}\|x_{i}^{m+1}-W_{m+1}z_{i}\|_{2}^{2} + C_{1}\|W_{m+1}\|_{*}.
\end{equation}
Problem (\ref{eq:batch}) can then be naturally adapted to $m+1$ views, and alternately optimizing view functions $\{W_{v}\}_{v=1}^{m+1}$ and subspace representations $\{z_{i}\}_{i=1}^{n}$ for several iterations will output the optimal multi-view model.

This naive approach to handling streaming views can be treated as a stochastic optimization strategy for solving Problem (\ref{eq:batch}). The view function for the new view can be efficiently discovered with the help of latent representations learned on past views; however, it is computationally expensive to upgrade the view functions on past views by re-launching Problem (\ref{eq:batch}), especially when the number of  views and the view function dimensions are large. 

\subsubsection{Streaming View Learning}

We begin the development of the streaming view learning algorithm by carefully investigating the view function. Note that any matrix $W_{v}\in \mathbb{R}^{D_{v}\times d}$ can be represented as the sum of rake-one matrices:
\begin{equation}\label{eq:svd_view_func}
W_{v} = \sum_{ij}\sigma_{ij}^{v}a_{i}^{v}(b_{j}^{v})^{T},
\end{equation}
where $span(a_{1}^{v}, a_{2}^{v}, \cdots)=\mathbb{R}^{D_{v}}$ and $span(b_{1}^{v}, b_{2}^{v}, \cdots)=\mathbb{R}^{d}$, and $\{\sigma^{v}_{ij}\}$ are the  coefficients to combine different subspaces.

%Since $W_{v}$ has been assumed to be low rank (i.e. $k_{v}\leq \min(D_{v}, d)$), $\sigma^{v}$ is encouraged to be sparse. Hence, in solving the view function $W_{v}$, the aim is to eliminate the rank one subspaces which have zero weights and focus on the remaining active rank one subspaces.

Based on the new formulation of the view function in Eq. (\ref{eq:svd_view_func}), Problem (\ref{eq:batch}) can be reformulated as
{\small
\begin{equation}\label{eq:new_batch}
\begin{split}
\min \; & \frac{1}{nm}\sum_{i=1}^{n}\sum_{v=1}^{m}\|x_{i}^{v}-A_{v}S_{v}B_{v}^{T}z_{i}\|_{2}^{2} + C_{1}\sum_{v=1}^{m}\|S_{v}\|_{*} + C_{2}\sum_{i=1}^{n}\|z_{i}\|_{2}^{2}\\
\text{w.r.t.} & \; \forall v\; A_{v}\in \mathbb{R}^{D_{v}\times k_{v}}, B_{v}\in \mathbb{R}^{d\times k_{v}}, S_{v} \in\mathbb{R}^{k_{v}\times k_{v}}; \forall i\; z_{i}\in\mathbb{R}^{d}\\
\text{s.t.}  &\; \forall v \; A_{v}^{T}A_{v} = I, \quad B_{v}^{T}B_{v}=I, 
\end{split}
\end{equation}
}\noindent
where $A_{v}$ and $B_{v}$ correspond to the  column and row spaces of $W_{v}$ respectively,  $S_{v}$ contains the weights to combine different rank-one subspaces, and $k_{v}$ indicates the number of active function subspaces on the $v$-th view.

Suppose that we already have  well trained view functions $\{(A_{v}, B_{v}, S_{v})\}_{v=1}^{m}$ over $m$ views. For the new $(m+1)$-th view, we can efficiently discover its view function $(A_{m+1}, B_{m+1}, S_{m+1})$ given the fixed latent representations $\{z_{i}\}_{i=1}^{n}$, 
{\small
\begin{equation}
\begin{split}
\min &\; \frac{1}{n}\sum_{i=1}^{n}\|x_{i}^{m+1}-A_{m+1}S_{m+1}B_{m+1}^{T}z_{i}\|_{2}^{2} + C_{1}\|S_{m+1}\|_{*} \\
\text{w.r.t.} & \quad A_{m+1} \in \mathbb{R}^{D_{m+1}\times k_{m+1}},  \; B_{m+1}\in \mathbb{R}^{d\times k_{m+1}},  \\
& \quad \;  S_{m+1} \in\mathbb{R}^{k_{m+1}\times k_{m+1}}\\
\text{s.t.}  &\quad A_{m+1}^{T}A_{m+1} = I, \quad B_{m+1}^{T}B_{m+1}=I.
\end{split}
\end{equation}
}\noindent
The remaining task is  to then upgrade the view functions on past $m$ views using the latest information. As mentioned above, completely re-training the model on past views is computationally expensive since a large number of variables need to be learned. Instead,  we propose to fine-tune the previously well-trained multi-view model using the following objective function:
{\small
\begin{equation}\label{eq:new_sz}
\begin{split}
\min\; & \frac{1}{n(m+1)}\sum_{i=1}^{n}\sum_{v=1}^{m+1}\|x_{i}^{v}-A_{v}S_{v}B_{v}^{T}z_{i}\|_{2}^{2} \\
&\quad \quad \quad \quad \quad  + C_{1}\sum_{v=1}^{m}\|S_{v}\|_{*} + C_{2}\sum_{i=1}^{n}\|z_{i}\|_{2}^{2}\\
\text{w.r.t.} & \; \forall v \; S_{v} \in\mathbb{R}^{k_{v}\times k_{v}}; \quad \forall i\; z_{i}\in\mathbb{R}^{d},\\
\end{split}
\end{equation}
}\noindent
where we have fixed the row and column spaces of view functions on multiple views and attempted to update view functions by adjusting their coefficients for  subspace combination. Since the view functions are now mainly determined by a set of smaller matrices $\{S_{v}\}_{v=1}^{m+1}$, where $S_{v}\in\mathbb{R}^{k_{v}\times k_{v}}$ with $k_{v}\ll \min(D_{v},d)$, solving Problem (\ref{eq:new_sz}) is often much cheaper than solving Problem (\ref{eq:new_batch}) or  (\ref{eq:batch}) with $m+1$ views. 

After solving or updating the view functions on $(m+1)$ views, the multi-view model  can then  process another new view. Meanwhile, the current multi-view model can be used to  predict the latent representation of a new multi-view example followed by  subsequent tasks.

\section{Optimization}

The proposed streaming view learning algorithm involves  optimization over latent representations $z=\{z_{i}\}_{i=1}^{m}$ and function subspaces on multiple views $\{(A_{v}, B_{v})\}_{v=1}^{m}$ and their corresponding combination weights $\{S_{v}\}_{v=1}^{m}$. In this section, we employ an alternating minimization strategy to optimize these variables. The whole optimization procedure is summarized in Algorithm 1.

\subsection{Optimization Over Latent Representations}

Fixing view functions $\{W_{v}=A_{v}S_{v}B_{v}^{T}\}_{v=1}^{m}$ on multiple views, the optimization problem w.r.t. the latent representation of the $i$-th example is 
\begin{equation}\label{eq:pro_zi}
\min_{z_{i}} \; \frac{1}{nm}\sum_{v=1}^{m}\|x_{i}^{v}-W_{v}z_{i}\|_{2}^{2} + C_{2}\|z_{i}\|_{2}^{2},
\end{equation}
which is easy to solve in a closed form.

\subsection{Optimization Over View Function Subspaces}
By fixing the latent representations, the view functions on multiple views can be independently optimized via
\begin{equation}
\min_{W_{v}} \; g(W_{v})+ C_{1}\|W_{v}\|_{*},
\end{equation}
where 
\begin{equation}
g(W_{v}) = \frac{1}{nm}  \sum_{i=1}^{n}\|x_{i}^{v}-W_{v}z_{i}\|_{2}^{2}.
\end{equation}
The proximal gradient descent method \citep{ji2009accelerated} has been widely used to solve this problem by reformulating it to,
{\small
\begin{equation}\label{eq:proximal}
\min_{W_{v}} \; \frac{\eta_{t}}{2}\|W_{v}-\big( W_{v}^{t-1} - \frac{1}{\eta_{t}}\nabla g(W_{v}^{t-1})\big) \|_{F}^{2} + C_{1}\|W_{v}\|_{*},
\end{equation}
}\noindent
where $\eta_{t}$ is the step size in the $t$-th iteration. It turns out that Problem (\ref{eq:proximal}) can be solved by singular value thresholding (SVT) \citep{cai2010singular},
\begin{equation}\label{eq:soft}
W_{v}^{t} = soft\big(W_{v}^{t-1} - \frac{1}{\eta_{t}}\nabla g(W_{v}^{t-1}), \frac{C_{1}}{\eta_{t}}\big),
\end{equation}
where $soft(X, C) = A(\Sigma-CI)_{+}B^{T}$ with singular value decomposition $X=A\Sigma B^{T}$ for $X$.

By operating Eq. (\ref{eq:soft}), we can obtain the view function subspace. However, Eq. (\ref{eq:soft}) requires accurate SVD over $W_{v}^{t-1} - \frac{1}{\eta_{t}}\nabla g(W_{v}^{t-1})$, which is computationally expensive given the large dimension of $W_{v}$. Recall that this step is only used to identify the view function subspaces, and the view function is more accurately discovered by optimizing the combination weight. Therefore, it is unnecessary to compute the SVD of $W_{v}^{t-1} - \frac{1}{\eta_{t}}\nabla g(W_{v}^{t-1})$ very accurately. We apply the power method \citep{halko2011finding} with several iterations to approximately calculate $W_{v}^{t-1} - \frac{1}{\eta_{t}}\nabla g(W_{v}^{t-1})\approx\widehat{A}_{v}\widehat{\Sigma}_{v}\widehat{B}_{v}^{T}$.

We initialize the optimization method by $(W_{v})_{0}=X_{v}Z^{T}(ZZ^{T})^{-1}$, and assume $(W_{v})_{0} = {A}_{v}\Sigma_{v}{B}_{v}^{T}$ is the reduced SVD of $(W_{v})_{0}$, where ${A}_{v} \in \mathbb{R}^{D_{v}\times k_{v}}$, ${B}_{v} \in \mathbb{R}^{d\times k_{v}}$ and $\Sigma_{v}\in \mathbb{R}^{k_{v}\times k_{v}}$ is diagonal. At each iteration, we calculate $W_{v}^{t-1} - \frac{1}{\eta_{t}}\nabla g(W_{v}^{t-1})\approx\widehat{A}_{v}\widehat{\Sigma}_{v}\widehat{B}_{v}^{T}$ using the cheaper power method and then filter out the singular vectors $\{\widetilde{A}_{v},\widetilde{B}_{v}\}$ with singular values greater than $C_{1}/\eta_{t}$. The column and row function subspaces can thus be discovered as the orthonormal bases of $span(A_{v}, \widetilde{A}_{v})$ and $span(B_{v}, \widetilde{B}_{v})$, respectively.

\subsection{Optimization Over Combination Weights}

Fixing the latent representations and the discovered view function subspaces, the optimization problem w.r.t. the combination weight $S_{v}$ on the $v$-th view is
\begin{equation}
\min_{S_{v}} \; h(S_{v})+ C_{1}\|S_{v}\|_{*},
\end{equation}
where 
\begin{equation}
h(W_{v}) = \frac{1}{nm}  \sum_{i=1}^{n}\|x_{i}^{v}-A_{v}S_{v}B_{v}^{T}z_{i}\|_{2}^{2}.
\end{equation}
Similarly, the proximal gradient technique can be applied to obtain an equivalent objective function,
{\small
\begin{equation}\label{eq:pro_sv}
\min_{S_{v}} \; \frac{\eta_{t}}{2}\|S_{v}-\big( S_{v}^{t-1} - \frac{1}{\eta_{t}}\nabla h(S_{v}^{t-1})\big) \|_{F}^{2} + C_{1}\|S_{v}\|_{*}.
\end{equation}
}\noindent
Since $S_{v}$ is a small $k_{v}\times k_{v}$ matrix, using the SVT method with an exact SVD operation to solve $S_{v}$ is feasible. 

\begin{algorithm}[tb]
   \caption{Streaming View Learning}
   \label{alg:lrml}
\begin{algorithmic}
   \STATE {\bfseries Input:} $\{X^{v}\}_{v=1}^{m+1}$, $\{z_{i}\}_{i=1}^{n}$, $\{A_{v}, S_{v}, B_{v}\}_{v=1}^{m}$

   \STATE {\bfseries PART 1- Solve new view function:}
   \STATE {\bfseries Initialize} $X_{v}Z^{T}(ZZ^{T})^{-1}={A}_{v}\Sigma_{v}{B}_{v}^{T}$, $v=m+1$
   \FOR {$t=1, \cdots, $}
    \STATE $(\widehat{A}_{v},\widehat{\Sigma}_{v},\widehat{B}_{v}^{T}) \leftarrow Power\big(W_{v}^{t-1} - \frac{1}{\eta_{t}}\nabla g(W_{v}^{t-1})\big)$
    \STATE $(\widetilde{A}_{v},\widetilde{B}_{v}) \leftarrow SVT(\widehat{A}_{v}\widehat{\Sigma}_{v}\widehat{B}_{v}^{T}, C_{1}/\eta_{t})$
    \STATE $A_{v}\leftarrow QR([A_{v}, \widetilde{A}_{v}])$ and $B_{v}\leftarrow QR([B_{v}, \widetilde{B}_{v}])$
    \STATE Solve $S_{v}$ via Problem (\ref{eq:pro_sv})
    \STATE $(A_{v}^{'}, S_{v}^{'}, B_{v}^{'}) \leftarrow SVD(S_{v})$
    \STATE $A_{v}\leftarrow A_{v}A_{v}^{'}$, $S_{v}\leftarrow S_{v}S_{v}^{'}$, $B_{v}\leftarrow B_{v}B_{v}^{'}$
    \STATE $W_{v}^{t}\leftarrow A_{v}S_{v}B_{v}^{T}$
   \ENDFOR
   \STATE {\bfseries PART 2- Upgrade view functions:}
   \FOR {$t=1, \cdots, $}
           \STATE $\forall \; i$ Solve $z_{i}$ via Problem (\ref{eq:pro_zi})
           \STATE $\forall \; v$ Solve $S_{v}$ via Problem (\ref{eq:pro_sv})
    \ENDFOR
   \STATE {\bfseries Return $\{A_{v}, S_{v}, B_{v}\}_{v=1}^{m+1}$}
\end{algorithmic}
\end{algorithm}

\section{Theoretical Analysis}
\label{sec:theory}
Here we conduct a theoretical analysis to reveal important properties of the proposed streaming view learning algorithm.

We use the following theorem to show that the latent representations $Z=[z_{1}, \cdots, z_{n}]$ become increasingly stable as streaming view learning progresses.
\begin{theorem}\label{the:z}
Given the latent representations $Z^{m-1}$ learned over $m-1$ views, and $Z^{m}$ learned over past $m-1$ views and the new $m$-th view (i.e.  $m$ views in total), we have $\|Z^{m}-Z^{m-1}\|_{F}=\mathcal{O}(1/m)$.
\end{theorem}

\begin{proof}
Given 
\begin{equation}
J_{m}(Z) = \frac{1}{m}\sum_{v=1}^{m}\ell(X^{v}, W^{v}, Z)+C_{2}\|Z\|_{F}^{2},
\end{equation}
we have
{\small
\begin{equation}\nonumber
\begin{split}
J_{m}(Z) - &J_{m-1}(Z) =  \frac{1}{m}\ell(X^{m}, W^{m}, Z) + \frac{1}{m}\sum_{v=1}^{m-1}\ell(X^{v}, W^{v}, Z)\\
- & \frac{1}{m-1}\sum_{v=1}^{m-1}\ell(X^{v}, W^{v}, Z) \\
= & \frac{1}{m}\ell(X^{m}, W^{m}, Z) - \frac{1}{m(m-1)}\sum_{v=1}^{m-1}\ell(X^{v}, W^{v}, Z).
\end{split}
\end{equation}
}\noindent
Since $\ell(\cdot)$ used in the algorithm is Lipschitz in its last argument, $J_{m}(Z) - J_{m-1}(Z)$ has a Lipschitz constant $\mathcal{O}(1/m)$. Assuming the Lipschitz constant of $J_{m}(Z) - J_{m-1}(Z)$ is $\theta_{m}$, we have
{\small
\begin{equation}\nonumber
\begin{split}
J_{m-1}(Z^{m}) - &J_{m-1}(Z^{m-1}) = J_{m-1}(Z^{m}) - J_{m}(Z^{m}) + J_{m}(Z^{m}) \\
-& J_{m}(Z^{m-1}) + J_{m}(Z^{m-1}) - J_{m-1}(Z^{m})\\
\leq & J_{m-1}(Z^{m}) - J_{m}(Z^{m}) +  J_{m}(Z^{m-1}) - J_{m-1}(Z^{m})\\
\leq & \theta_{m}\|Z^{m}-Z^{m-1}\|_{F}.
\end{split}
\end{equation}
}\noindent
Since $Z^{m-1}$ is the minimum of $J_{m-1}(Z)$, we have
\begin{equation}
J_{m-1}(Z^{m})-J_{m-1}(Z^{m-1}) \geq 2C_{2}\|Z^{m}-Z^{m-1}\|_{F}^{2}.
\end{equation}
Combining the above results, we have,
\begin{equation}
\|Z^{m}-Z^{m-1}\|_{F} \leq \frac{\theta_{m}}{2C_{2}}=\mathcal{O}(1/m),
\end{equation}
which completes the proof.
\end{proof}

Theorem \ref{the:z} reveals that the streaming views are helpful for deriving a stable multi-view model. We next analyze the optimality of the discovered view function for the newly arrived view using the following theorem.

\begin{theorem}\label{the:subspace}
The  optimization steps of Part 1 in Algorithm 1 can guarantee that the  solved function subspaces of the new view converge to a stationary point.
\end{theorem}

The proof of Theorem \ref{the:subspace} is listed in supplementary material due to page limitation. This remarkable result shows that the proposed algorithm can efficiently discover the optimal view function subspaces. We next analyze the influence of  perturbation of the latent representations on the view function estimation by the following theorem, whose detailed proof is listed in supplementary material.

\begin{theorem}\label{the:diff_w}
Fixing $\|X^{v}\|\leq \Upsilon$ for each view. Given latent representation $Z$ with $\|Z\|_{F}\leq \Omega$, the optimal view function on the $v$-th view is denoted as $W_{v}$. For $\widetilde{Z}$ with $\|\widetilde{Z}-Z\|_{F}\leq \epsilon$, the optimal view function on the $v$-th view is defined as $W_{v}^{'}$. Suppose the smallest eigenvalue of $\widetilde{Z}\widetilde{Z}^{T}$ is lower bounded by $\lambda>0$, and both the rank of $W_{v}$ and $W_{v}^{'}$ are lower than $k$. The following error bound holds
\begin{equation}
\|W_{v}^{'}-W_{v}\|_{F} \leq  \frac{1}{\lambda} \big(\Upsilon^{2}\frac{2\epsilon\Omega+\epsilon^{2}}{C_{1}} + \epsilon\Upsilon+ 2\sqrt{k+1}\big)
\end{equation}
\end{theorem}

This theoretical analysis allows us to summarize as follows. For the newly arrived view, the proposed algorithm is guaranteed to discover its optimal view function based on the convergence analysis in Theorem \ref{the:subspace}. According to Theorem \ref{the:z}, the learned latent representation will become increasingly stable with more streaming views. Hence, given a small perturbation on latent representation matrix $Z$, the difference between the target view functions is also small. Most importantly, according to perturbation theory \citep{li1998relative}, it is thus reasonable to assume that the view function subspaces are approximately consistent given the bounded perturbation on the matrix; therefore, it is feasible to simply fine tune the combination weights of these subspaces for better reconstruction. On the other hand, if the number of past views is small, we can use the standard multi-view learning algorithm to re-train the model over past views and the new views together with acceptable resource cost.

\section{Experiments}

We  next evaluated the proposed SVL algorithms for clustering and classification of real-world datasets. The SVL algorithm was compared to  canonical correlation analysis (CCA) \citep{hardoon2004canonical}, the convex multi-view subspace learning algorithm (MCSL) \citep{white2012convex}, the factorized latent sparse with structured sparsity algorithm (FLSSS) \citep{jia2010factorized}, and the shared Gaussian process latent variable model (sGPLVM) \citep{shon2005learning}. Since these comparison algorithms were not designed for the streaming view setting,  we adapted the  algorithms for fair comparison such that they  employed the idea of multi-view learning to handle new views. Specifically, for each multi-view comparison algorithm,  the outputs of the well-trained multi-view model over past views were treated as temporary views, which were then combined with the newly arrived view to train a new multi-view model. Note that we did not adopt the trick to completely re-train multi-view learning algorithms using past views and new views simultaneously, since it is infeasible for practical applications considering the intensive computational cost.

The real-world datasets used in experiments were the \textit{Handwritten Numerals} and \textit{PASCAL VOC'07} datasets. The  Handwritten Numerals dataset is composed of $2,000$ data points in 0 to 9 ten-digit classes, where each class contains 200 data points. Six types of features are employed to describe the data: Fourier coefficients of the character shapes (FOU),  profile correlations (FAC),  Karhunen-Lo$\grave{e}$ve coefficients (KAR),  pixel averages in $2\times 3$ windows (PIX),  Zernike moments (ZER), and  morphological  features (MOR). The \textit{PASCAL VOC'07} dataset contains around $10,000$ images, each of which was annotated with 20 categories. Sixteen types of features have been used to describe each image including GIST, image tags, 6 color histograms (RGB, LAB, and HSV over single-scale or multi-scale images), and 8 bag-of-features descriptors (SIFT and hue densely extracted or for Harries-Laplacian interest points on single-scale or multi-scale images).

\subsection{Multi-view Clustering and Classification}
For each algorithm, half of the total views were used for  initialization to well train a base multi-view model, and then multi-view learning was conducted with the streaming views. We fixed the dimension of the latent subspace as $100$ for different algorithms. Based on the multi-view example subspaces learned through the proposed SVL algorithm and its comparison algorithms, the k-means and SVM methods were launched for  subsequent clustering and classification, respectively. Clustering performance was assessed by normalized mutual information (NMI) and accuracy (ACC), while classification performance was measured using mean  averaged precision (mAP).

\begin{table*}[!tbh] %\setlength{\tabcolsep}{3pt}
\caption{NMI on the \textit{Handwritten Numerals} dataset. The number following each feature denotes  the number of new views  processed. `Single' implies directly launching k-means on the current view.}
%\vskip 0.1in
\label{tab:clustering}
\begin{center}
%\begin{small}
\begin{tabular}{ccccccc}
\hline
 Algorithm  & FAC &  FOU  & KAR (0) & MOR (1)  & PIX (2) & ZER (3)\\
\hline
\hline
Single   & $0.679\pm 0.032$  & $0.547\pm 0.028$ & $0.666\pm 0.030$ & $0.643\pm 0.034$ & $0.703\pm 0.040$ & $0.512\pm 0.025$ \\
CCA       & - & - & $0.755\pm 0.039$ & $0.777\pm 0.038$ & $0.772\pm 0.061$ & $0.799\pm 0.038$ \\
FLSSS     & - & - & $0.833\pm 0.047$ & $0.837\pm 0.029$ & $0.830\pm 0.028$ & $0.840\pm 0.027$ \\
sGPLVM  & - & - & $0.785\pm 0.044$ & $0.794\pm 0.021$ & $0.799\pm 0.056$ & $0.827\pm 0.055$ \\
MCSL     & - & - & $0.798\pm 0.028$ & $0.805\pm 0.047$ & $0.814\pm 0.069$ & $0.815\pm 0.026$ \\
SVL       & - & - & $0.826\pm 0.049$ & $0.840\pm 0.044$ & $0.866\pm 0.037$ & $0.871\pm 0.042$ \\
\hline
\end{tabular}\vskip -0.2in
%\end{small}
\end{center}\vskip -0.1in
\end{table*}

\begin{table*}[!tbh] \vskip -0.2in%\setlength{\tabcolsep}{3pt}
\caption{ACC on the \textit{Handwritten Numerals} dataset. The number following each feature  denotes the number of new views processed. `Single' implies directly launching k-means on the current view.}
%\vskip 0.1in
\label{tab:clustering_acc}
\begin{center}
%\begin{small}
\begin{tabular}{ccccccc}
\hline
 Algorithm  & FAC   & FOU  & KAR (0) & MOR (1) &  PIX (2) & ZER (3) \\
\hline
\hline
Single    & $0.707\pm 0.065$ & $0.556\pm 0.062$  & $0.689\pm 0.051$ & $0.614\pm 0.058$ & $0.694\pm 0.067$ & $0.534\pm 0.052$ \\
CCA       & - & - & $0.709\pm 0.051$ & $0.710\pm 0.015$ & $0.706\pm 0.037$ & $0.809\pm 0.065$ \\
FLSSS     & - & - & $0.819\pm 0.038$ & $0.831\pm 0.035$ & $0.851\pm 0.027$ & $0.849\pm 0.045$ \\
sGPLVM  & - & - & $0.777\pm 0.052$ & $0.796\pm 0.035$ & $0.788\pm 0.053$ & $0.805\pm 0.058$ \\
MCSL     & - & - & $0.790\pm 0.045$ & $0.804\pm 0.055$ & $0.804\pm 0.036$ & $0.836\pm 0.050$ \\
SVL       & - & - & $0.813\pm 0.052$ & $0.816\pm 0.032$ & $0.908\pm 0.050$ & $0.927\pm 0.050$ \\
\hline
\end{tabular}
%\end{small}
\end{center}\vskip -0.1in
\end{table*}

\begin{table}[th] \vskip -0.2in%\setlength{\tabcolsep}{3pt}
\caption{mAP on the \textit{PASCAL VOC'07} dataset. `8' means that multi-view learning begins with 8 views, while the following number denotes the number of new views processed.}
%\vskip 0.1in
\label{tab:classification}
\begin{center}
%\begin{small}
\begin{tabular}{cccccc}
\hline
 Algorithm  & $8(0)$   & $8(2)$   & $8(4)$  & $8(6)$  & $8(8)$ \\
\hline
\hline
CCA       & $0.314$ & $0.335$ & $0.347$ & $0.357$ & $0.368$  \\
FLSSS     & $0.507$ & $0.518$ & $0.521$ & $0.532$ & $0.539$  \\
sGPLVM  & $0.441$ & $0.458$ & $0.473$ & $0.487$ & $0.487$  \\
MCSL      & $0.448$ & $0.463$ & $0.469$ & $0.475$ & $0.480$  \\
SVL        & $0.554$  & $0.558$ & $0.562$ & $0.564$ & $0.571$  \\
\hline
\end{tabular}
%\end{small}
\end{center}\vskip -0.15in
\end{table}

The performance  of  different algorithms with respect to the progress of streaming view learning of the clustering task are shown in Tables \ref{tab:clustering} and \ref{tab:clustering_acc}. Classification results under  similar settings are presented in Table \ref{tab:classification}. In each table, multi-view learning algorithms learn the views from the left to the right column in a streaming manner, such that the results presented on the right side have already been helped by the views on the left. The different multi-view learning algorithms  consistently improve  clustering and classification performance when more new view information becomes available. Although the base multi-view learning model of the proposed SVL algorithm only achieves comparable or slightly inferior performance to that of comparison algorithms, SVL  significantly improves its performance by optimally learning new view functions and upgrading past view functions, such that the advantages of SVL becomes more obvious with increasing of numbers of new views processed. Specifically, in the fifth column of Table \ref{tab:clustering}, the NMI of SVL  improves about $20\%$ over that of single-view algorithm and $5\%$ over that of multi-view FLSSS algorithm. 
%It can be seen from Table \ref{tab:classification} that most of the comparison algorithms were computationally expensive to train multi-view models with 8 views,  thus making it impractical to launch the multi-view learning algorithms  again by simultaneously using the past 8 views and the new views. 
On the \textit{PASCAL VOC'07} dataset, training multi-view model over 8 views is already computational expensive, let alone completely re-training with new views.

\subsection{Algorithm Analysis}

We next varied the dimensionality $d$ of latent representations. The clustering performance of SVL on the \textit{Handwritten Numerals} dataset is presented in Table \ref{tab:d}.  The performance of the lower-dimensional latent representations is limited, whereas with  increased $d$ the latent representations  have  more power to describe  multi-view examples, and  the SVL algorithm  achieves stable performance.

\begin{table}[th] \vskip -0.2in\setlength{\tabcolsep}{3pt}
\caption{NMI of SVL with different dimensionalities of latent representations on the \textit{Handwritten Numerals} dataset.}
%\vskip 0.1in
\label{tab:d}
\begin{center}
%\begin{small}
\begin{tabular}{ccccccc}
\hline
 $d$  & FAC   & FOU  & KAR (0) & MOR (1) &  PIX (2) & ZER (3) \\
\hline
\hline
10     & - & - & $0.655$ & $0.655$ & $0.688$ & $0.705$\\
20     & - & - & $0.680$ & $0.762$ & $0.771$ & $0.772$\\
50     & - & - & $0.755$ & $0.778$ & $0.809$ & $0.838$\\
100     & - & - & $0.826$ & $0.840$ & $0.866$ & $0.871$\\
150     & - & - & $0.829$ & $0.834$ & $0.865$ & $0.888$\\
\hline
\end{tabular}
%\end{small}
\end{center}
\end{table}

\begin{figure}[!thb]\vskip -0.2in
\begin{center}
   \includegraphics[width=0.8\columnwidth]{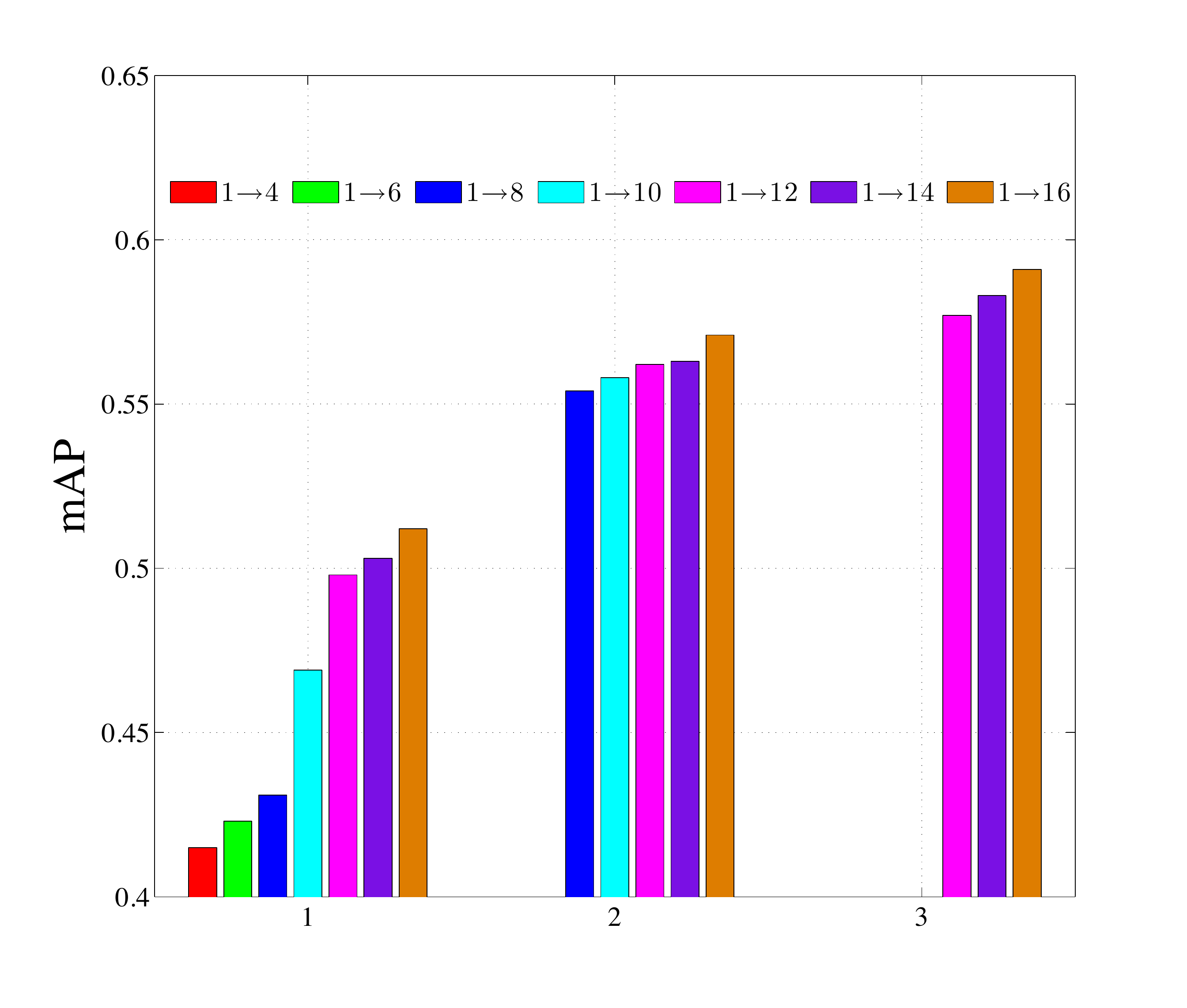}
\end{center}\vskip -0.2in
   \caption{Classification performance of SVL using different numbers of views for initialization on the \textit{PASCAL VOC'07} dataset.}
\label{fig:m}\vskip -0.2in
\end{figure}

\begin{figure}[!thb]\vskip -0.2in
\begin{center}
   \includegraphics[width=0.8\columnwidth]{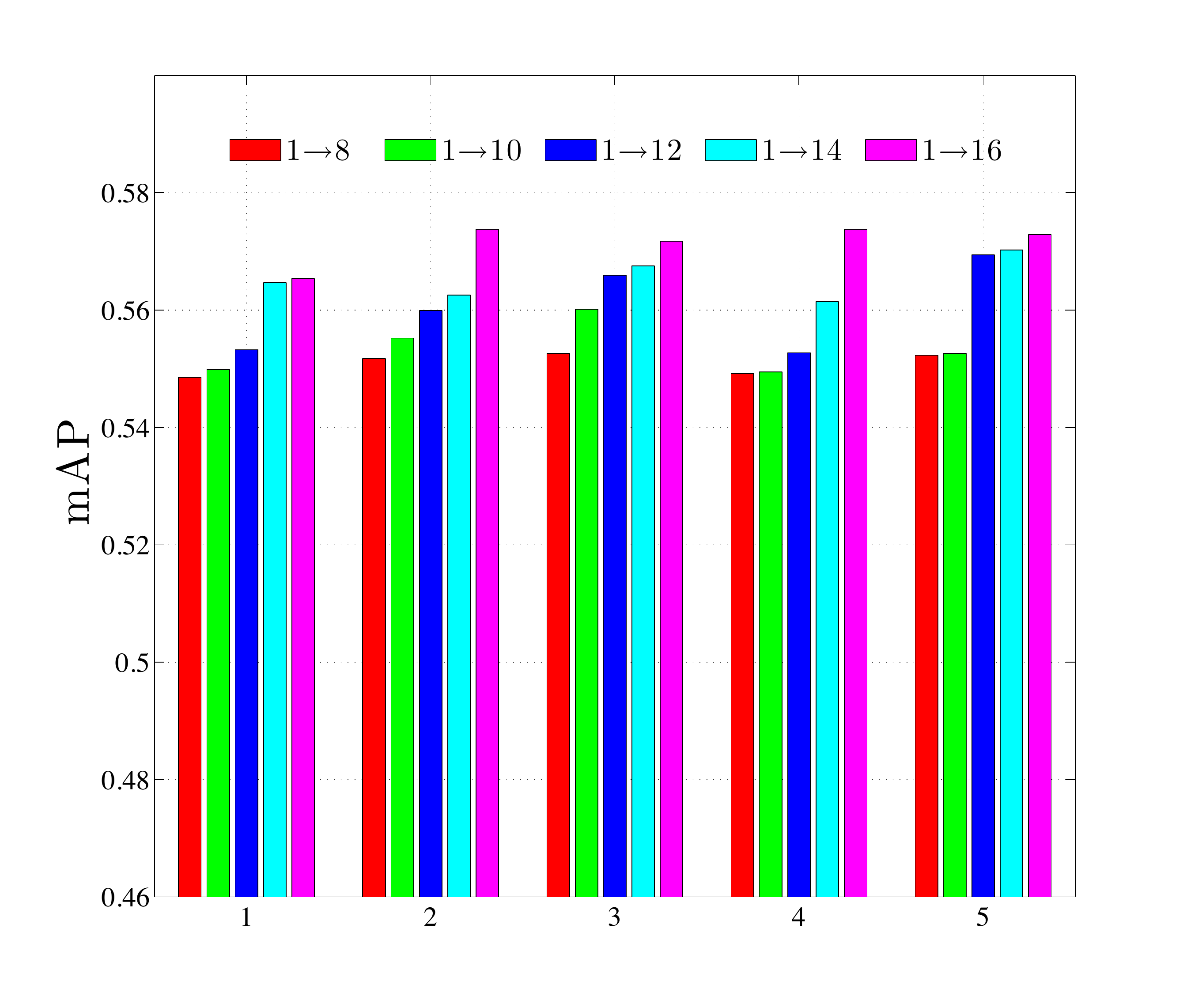}
\end{center}\vskip -0.2in
   \caption{Classification Performance of SVL with distinct view orders on the PASCAL VOC'07 dataset. }\vskip -0.1in
\label{fig:order}
\end{figure}

To examine the influence of the number of views used to initialize the base multi-view SVL model, we started SVL with different  numbers of views on the \textit{PASCAL VOC'07} dataset. The variability in performance  is presented in Figure \ref{fig:m}. If the base multi-view model was initialized with insufficient views, the resulting performance was limited (see the first group in Figure \ref{fig:m}). This is due to the large estimation error of the  functions over the views used to initialize the model. Conversely, if the multi-view model over past views was already well trained, we  easily applied SVL to extend the model to handle new views without intensive computational cost, whilst also guaranteeing  stable performance improvements. These phenomena are  consistent with our theoretical analyses. 

Finally, we examined the influence of the order of streaming views on the learning performance of SVL, and  the classification results are shown in Figure \ref{fig:order}. For each group in Figure \ref{fig:order}, the view order is randomly determined. It can be seen that although classification performance variations is diverse, the resulting performances are roughly equivalent with distinct streaming view orders.

\section{Conclusions}
\vspace{-0.05in}
Here  we investigate the streaming view problem in multi-view learning, in which views arrive in a streaming manner. Instead of discarding the multi-view model trained well over past views, we regard the subspaces of the well-trained view functions as stable and fine-tune the weights for function subspaces combinations while processing new views. In this way, the resulting SVL algorithm can efficiently learn view functions for the new view and update view functions for past views. The convergence issue of the proposed algorithm is theoretically studied, and the influence of streaming views on the multi-view model is addressed. Comprehensive experiments conducted on real-world datasets demonstrate the significance of studying the streaming view problem and the effectiveness of the proposed algorithm.

\bibliography{SVL_arxiv.bbl}

\end{document}